\documentclass[10pt,twocolumn,letterpaper]{article}

\usepackage{cvpr}
\usepackage{times}
\usepackage{epsfig}
\usepackage{graphicx}
\usepackage{amsmath}
\usepackage{amssymb}
\usepackage{multirow}
\usepackage{listings}
\usepackage{subcaption}
\usepackage{gensymb}
\usepackage{placeins}
\usepackage{amsthm}
\usepackage{bm}

\usepackage[pagebackref=true,breaklinks=true,letterpaper=true,colorlinks,bookmarks=false]{hyperref}

\cvprfinalcopy 

\def\cvprPaperID{2619} 

\ifcvprfinal\fi
\begin{document}

\title{A Minimal Solution for Two-view Focal-length Estimation using \\Two Affine Correspondences}

\newtheorem{mylem}{Lemma}

\author{Daniel Barath, Tekla Toth, and Levente Hajder\\
Machine Perception Research Laboratory\\
MTA SZTAKI, Budapest, Hungary\\
{\tt\small \{barath.daniel,hajder.levente\}@sztaki.mta.hu}}

\maketitle
\thispagestyle{empty}

\begin{abstract}
	A minimal solution using two affine correspondences is presented to estimate the common focal length and the fundamental matrix between two semi-calibrated cameras -- known intrinsic parameters except a common focal length. To the best of our knowledge, this problem is unsolved. The proposed approach extends point correspondence-based techniques with linear constraints derived from local affine transformations. The obtained multivariate polynomial system is efficiently solved by the hidden-variable technique. Observing the geometry of local affinities, we introduce novel conditions eliminating invalid roots. To select the best one out of the remaining candidates, a root selection technique is proposed outperforming the recent ones especially in case of high-level noise. The proposed 2-point algorithm is validated on both synthetic data and 104 publicly available real image pairs. A Matlab implementation of the proposed solution is included in the paper. 
\end{abstract}

\section{Introduction}

The recovery of camera parameters and scene structure have been studied for over two decades since several applications, such as 3D vision from multiple views~\cite{Hartley2003}, are heavily dependent on the quality of the camera calibration. In particular, two major calibration types can be considered: aiming at the determination of the intrinsic and/or extrinsic parameters. The former ones include focal lengths, principal point, aspect ratio, and non-perspective distortion parameters, while the extrinsic parameters are the relative pose.
%
%
Assuming two cameras with unknown extrinsic and \textit{a priori} intrinsic parameters except a common focal length is called the \textit{semi-calibrated case}~\cite{li2006simple}. It leads to the \textit{unknown focal-length problem}: estimation of the relative motion and common focal length, simultaneously.
The semi-calibrated case is realistic since (1) the aspect ratio is determined by the shape of the pixels on the sensors, it is usually 1:1; (2) the principal point is close to the center of the image, thus it is a reasonable approximation and (3) the distortion can be omitted if narrow field-of-view lenses are applied.
Considering solely the locations of point pairs makes the problem solvable using at least six point  pairs~\cite{li2006simple,stewenius2008minimal,torii2010six}. The objective of this paper is to \textit{solve the problem exploiting only two local affine transformations}.  

In general, 3D vision approaches~\cite{Hartley2003} including state-of-the-art structure-from-motion pipelines~\cite{Agarwal2011,BodisCVPR2014,Frahm2010,Moulon2013} apply a robust estimator, e.g.\ RANSAC~\cite{RANSAC}, augmented with a minimal method, such as the five~\cite{nister2004efficient} or six-point~\cite{li2006simple} algorithm as an engine. Selecting a method exploiting as few point pairs as possible gains accuracy and drastically reduces the processing time. Benefiting from estimators which use less input data, the understanding of low-textured environment becomes significantly easier~\cite{Raposo2016}. Moreover, minimal methods are advantageous from theoretical point-of-view leading to deeper understanding. 

Local affine transformations represent the warp between the infinitely close vicinities of corresponding point pairs~\cite{koser2009geometric} and have been investigated for a decade. Their application field includes homography~\cite{Barath2016novel} and surface normal~\cite{koser2009geometric,Barath2015} estimation; recovery of the epipoles~\cite{Bentolila2014}; triangulation of points in 3D~\cite{koser2009geometric}; camera pose estimation~\cite{Koser2008}; structure-from-motion~\cite{Raposo2016}. In practice, local affinities can be accurately retrieved~\cite{barath2016accurate,mikolajczyk2005comparison} using e.g.\ affine-covariant feature detectors, such Affine-SIFT~\cite{Morel2009} and Hessian-Affine~\cite{mikolajczyk2002affine}.
To the best of our knowledge, no paper has dealt with the unknown focal length problem using local affine transformations. 

This paper proposes two novel linear constraints describing the relationship between local affinities and epipolar geometry. Forming a multivariate polynomial system and solving it by the \textit{hidden-variable technique}~\cite{cox2006using}, the proposed method is efficient and estimates the focal length and the relative motion using only two affinities. In order to eliminate invalid roots, a novel condition is introduced investigating the geometry of local affinities. To select the best candidate out of the remaining ones, we propose a root selection technique which is as accurate as the state-of-the-art for small noise and outperforms it for high-level noise.   

\section{Preliminaries and Notation}

\noindent
\textbf{Epipolar geometry.} 
Assume two perspective cameras with a common intrinsic camera matrix $\mathbf{K}$ to be known. Fundamental and essential matrices~\cite{Hartley2003} are as follows:
\begin{equation*}
	\mathbf{F} = \begin{bmatrix}
    	f_1 & f_2 & f_3 \\
    	f_4 & f_5 & f_6 \\
    	f_7 & f_8 & f_9
   	\end{bmatrix}, \quad
	\mathbf{E} = \begin{bmatrix}
    	e_1 & e_2 & e_3 \\
    	e_4 & e_5 & e_6 \\
    	e_7 & e_8 & e_9
   	\end{bmatrix}.
\end{equation*}
If the cameras are calibrated ($\mathbf{K}$ is known) matrix $\mathbf{F}$ can be transformed to be an essential matrix $\mathbf{E}$ as follows:
\begin{equation}
	\mathbf{E} = \textbf{K}^{T} \mathbf{F} \mathbf{K}.
\end{equation}

The epipolar relationship of corresponding point pair $\mathbf{p}_1$ and $\mathbf{p}_2$ are described by $\mathbf{F}$ as
\begin{equation}
	\label{eq:epipolar_constraint}
	\mathbf{p}_2^T \mathbf{F} \mathbf{p}_1 = 0.
\end{equation}
A valid fundamental matrix must satisfy singularity constraint $\det(\mathbf{F}) = 0$.
Considering this cubic constraint and the fact that a fundamental matrix is defined up to an arbitrary scale, its degrees-of-freedom is reduced to seven. Thus seven point pairs are enough for the estimation. 

As the essential matrix encapsulates the full camera motion, the orientation and direction of the translation, it has five degrees-of-freedom. The two additional constraints are described by the well-known trace constraint~\cite{li2006simple} as 
\begin{equation}
	\label{eq:trace_constraint}
	2 \mathbf{E} \mathbf{E}^T \mathbf{E} - \textrm{tr}(\mathbf{E} \mathbf{E}^T) \textbf{E} = 0.
\end{equation}
Even though Eq.~\ref{eq:trace_constraint} yields nine polynomial equations for $\mathbf{E}$, only two of them are algebraically independent. 

Semi-calibrated case is assumed in this paper as only the common focal-length $f$ is considered to be unknown. Without loss of generality, the intrinsic camera matrix is
$
	\mathbf{K} = \mathbf{K}^T= \text{diag}(f, f, 1), 
$ 
where $f$ is the unknown focal-length. In order to replace $\mathbf{E}$ with $\mathbf{F}$ in Eq.~\ref{eq:trace_constraint} we define matrix $\mathbf{Q}$ as follows:
\begin{equation}
	\mathbf{Q} = \text{diag} \left( 1,1,\tau \right),
\quad    \tau = f^{-2}.
\end{equation} 
%


\noindent
Due to the fact that $\mathbf{K}$ is non-singular,  and trace($\mathbf{E} \mathbf{E}^T$) identifies a
scalar value,  Eq.~\ref{eq:trace_constraint} can be simplified by multiplying with $\mathbf{K}^{-T}$ and $\mathbf{K}^{-1}$ from the left and the right sides, respectively. Moreover, trace is invariant under cyclic permutations. As a consequence, Eq.~\ref{eq:trace_constraint} is written as~\cite{Kukelova2008BMVC,Pernek2013PRL}
\begin{equation}
	\label{eq:focal_trace_constraint}
	2 \mathbf{F} \mathbf{Q} \mathbf{F}^T \mathbf{Q} \mathbf{F} - \textrm{tr}(\mathbf{F} \mathbf{Q} \mathbf{F}^T \mathbf{Q}) \mathbf{F} = 0.
\end{equation}
This relationship will help us to recover the focal length and the fundamental matrix using two affine correspondences.

\noindent
\textbf{An affine correspondence} $(\mathbf{p}_1, \mathbf{p}_2, \mathbf{A})$ consists of a corresponding point pair and the related local affinity $\mathbf{A}$ transforming the vicinity of point $\mathbf{p}_1$ to that of $\mathbf{p}_2$. In the rest of the paper, $\mathbf{A}$ is considered as its left $2 \times 2$ submatrix
\begin{equation*}
	\mathbf{A} = \begin{bmatrix} a_1 & a_2 \\
    	a_3 & a_4
    \end{bmatrix}
\end{equation*}
since the third column -- the translation part -- is determined by the point locations. 

We use the \textbf{hidden variable technique} in the proposed method. It is a resultant technique in algebraic geometry for the elimination of variables from a multivariate polynomial system~\cite{cox2006using}. Suppose that $m$ polynomial equations in $n$ variables are given. In brief, one can assume an unknown variable as a parameter and rewrite the equation system as $\textbf{C}(y_1) \textbf{x} = 0$, where $\textbf{C}$ is a coefficient matrix depending on the unknown $y_1$ (hidden variable) and vector $\textbf{x}$ is the vector of $n - 1$ unknowns. If the number of equations equals to that of the unknown monomials in $\textbf{x}$, i.e. matrix $\textbf{C}$ is square, the non-trivial solution can be carried out as $\det(\textbf{C}(y_1)) = 0$. Solving the resultant equation for $y_1$ and back-substituting it, the whole system is solved.

\section{Focal-length using Two Correspondences}

This section aims the recovery of the unknown focal length and fundamental matrix using two affine correspondences. First, the connection between the fundamental matrix and local affinity is introduced, then we discuss the estimation technique.   \vspace{5px}

\noindent \vspace{2px}
\textbf{3.1. Exploiting a Local Affine Transformation}

Suppose that an affine correspondence $(\mathbf{p_1}, \mathbf{p_2}, \mathbf{A})$ and fundamental matrix $\mathbf{F}$ are known. It is trivial that every affine transformation preserves the direction of the lines going through points $\mathbf{p_1}$ and $\mathbf{p_2}$ on the first and second images. As a consequence, the link between directions $\mathbf{v_1}$ and $\mathbf{v_2}$ of epipolar lines can be described~\cite{barath2016accurate} by affine transformation $\mathbf{A}$ as
\begin{equation}
	\label{eq:a_v1_par_v2}
	\mathbf{A} \mathbf{v_1} \parallel \mathbf{v_2}.
\end{equation}




Reformulating Eq.~\ref{eq:a_v1_par_v2} using the well-known fact from Computer Graphics~\cite{turkowski1990transformations} leads to $\mathbf{A}^{-T} \mathbf{R}^{90} \mathbf{v_1} = \beta \mathbf{R}^{90} \mathbf{v_2}$, 
where matrix $\mathbf{R}^{90}$ is a 2D orthonormal (rotation) matrix rotating with $90$ degrees and $\beta$ is an unknown scale. Vectors $\mathbf{R}^{90} \mathbf{v_1}$ and $\mathbf{R}^{90} \mathbf{v_2}$ are the line normals $\mathbf{n_1}$ and $\mathbf{n_2}$ as 
\begin{equation}
	\label{eq:a_n1_eq_n2}
	\mathbf{A}^{-T} \mathbf{n_1} = \beta \mathbf{n_2}.
\end{equation}
In Appendix~\ref{appendix:proof_affine_constraints}, it is proven that $\beta$ is equal to $-1$ if $\mathbf{n_1}$ and $\mathbf{n_2}$ are calculated from the fundamental matrix using relationships $\mathbf{F} \mathbf{n_1}$ and $\mathbf{F}^T \mathbf{n_2}$ and they are \textit{not normalized}. In brief, it is given as the distance ratio of neighboring epipolar lines on the two images. For the case when the normals are not normalized -- the original scale has not been changed --, $\beta$ is only a scale inverting the directions. 

Normals are expressed from $\mathbf{F}$ as the first two coordinates of the epipolar lines: $\mathbf{n}_1 = (\mathbf{l}_1)_{(1:2)} = (\mathbf{F}^T \mathbf{p}_2)_{(1:2)}$ and $\mathbf{n}_2 = (\mathbf{l}_2)_{(1:2)} = (\mathbf{F} \mathbf{p}_1)_{(1:2)}$~\cite{Hartley2003}, where the lower indices select a subvector. Therefore, Eq.~\ref{eq:a_n1_eq_n2} is written as
\begin{equation}
	\mathbf{A}^{-T} (\mathbf{F}^T \mathbf{p}_2)_{(1:2)} = -(\mathbf{F} \mathbf{p}_1)_{(1:2)}
\end{equation}
and forms a system of linear equations consisting of two equations as follows:
\begin{eqnarray}
	\label{eq:affine_constraint_1}
	(u_2 + a_1 u_1) f_1 + a_1 v_1 f_2 + a_1 f_3 + (v_2 + a_3 u_1) f_4 + \nonumber \\ a_3 v_1 f_5 + a_3 f_6 + f_7 = 0 \\
	\label{eq:affine_constraint_2}
	a_2 u_1 f_1 + (u_2 + a_2 v_1) f_2 + a_2 f_3 + a_4 u_1 f_4 + \nonumber \\ (v_2 + a_4 v_1) f_5 + a_4 f_6 + f_8 = 0.
\end{eqnarray}
Thus each local affine transformation \textit{reduces the degrees-of-freedom by two}. 

\vspace{5px}

\noindent \vspace{2px}
\textbf{3.2. Two-point Solver}

Suppose that two affine correspondences ($\mathbf{p}_1^{1}$, $\mathbf{p}_2^{1}$, $\mathbf{A}^{1}$) and ($\mathbf{p}_1^{2}$, $\mathbf{p}_2^{2}$, $\mathbf{A}^{2}$) are given. Coefficient matrix 
\begin{equation*}
	\resizebox{1.0\columnwidth}{!}{
  	$\mathbf{C}^i = 
    \begin{bmatrix}
    	u_2 + a_1 u_1 & a_1 v_1 & a_1 & v_2 + a_3 u_1 & a_3 v_1 & a_3 & 1 & 0 & 0 \\
    	a_2 u_1 & u_2 + a_2 v_1 & a_2 & a_4 u_1 & v_2 + a_4 v_1 & a_4 & 0 & 1 & 0 \\
        u_1 u_2 & v_1 u_2 & u_2 & u_1 v_2 & v_1 v_2 & v_2 & u_1 & v_1 & 1
	\end{bmatrix}$
    }
\end{equation*}
related to the $i$-th ($i \in \{1,2\}$) correspondence is formed as the combination of Eqs.~\ref{eq:epipolar_constraint},~\ref{eq:affine_constraint_1},~\ref{eq:affine_constraint_2} and satisfies formula $\mathbf{C}^i \mathbf{x} = 0$, where $\mathbf{x} = \begin{bmatrix} f_1 & f_2 & f_3 & f_4 & f_5 & f_6 & f_7 & f_8 & f_9 \end{bmatrix}^T$ is the vector of unknown elements of the fundamental matrix. We denote the concatenated coefficient matrix of both correspondences as follows:
\begin{equation}
	\mathbf{C} = \begin{bmatrix}
    	\mathbf{C}^1 \\
        \mathbf{C}^2
    \end{bmatrix}.
\end{equation}
It is of size $6 \times 9$, therefore, its left null space is three-dimensional. The solution is carried out as
\begin{equation}
	\label{eq:null_space}
	\mathbf{x} = \alpha \mathbf{a} + \beta \mathbf{b} + \gamma \mathbf{c},  
\end{equation}
where $\mathbf{a}$, $\mathbf{b}$ and $\mathbf{c}$ are the singular vectors and $\alpha$, $\beta$, $\gamma$ are unknown non-zero scalar values. 

Remember that only the common focal length is unknown from the intrinsic parameters, therefore, we are able to exploit the trace constraint. Eq.~\ref{eq:focal_trace_constraint} yields ten cubic equations for four unknowns $\alpha$, $\beta$, $\gamma$ and $\tau$, where $\tau = f^{-2}$ encapsulates the unknown focal length.
We consider $\tau$ as the hidden variable and form coefficient matrix $\mathbf{C}(\tau)$ w.r.t.\ the other three ones -- thus the rows of $\mathbf{C}(\tau)$ are univariate polynomials with variable $\tau$. Even though $\alpha$, $\beta$ and $\gamma$ are defined up to a common scale, we do not fix this scale in order to keep the homogenity of the system. The monomials of this polynomial system are as 
$
	\mathbf{y} = [ \alpha^3 \;\; \alpha^2 \beta \;\; \alpha^2 \gamma  \;\; \alpha \beta^2 \;\; \alpha \beta \gamma \;\; \alpha \gamma^2 \;\; \beta^3 \;\; \beta^2 \gamma \;\; \beta \gamma^2 \;\; \gamma^3 ]^T
$.
Table~\ref{tab:coefficient_table} demonstrates the coefficient matrix. 

Since the scale of monomial vector $\mathbf{x}$ has not been fixed, the non-trivial solution of equation $\mathbf{C}(\tau) \mathbf{y} = 0$ is when the determinant vanishes as 
\begin{equation}
	\label{eq:determinant_hidden}
    \det(\mathbf{C}(\tau)) = 0.
\end{equation}
Therefore, the hidden-variable resultant -- a polynomial of the hidden variable -- is $\det(\mathbf{C}(\tau))$. As the current problem is fairly similar to that of~\cite{li2006simple}, we adopt the proposed algorithm. It is proved that $\det(\mathbf{C}(\tau))$ is actually a 15-th degree polynomial and it obtains the candidate values for $\tau$. Then the solution for $\alpha$, $\beta$, $\gamma$ and $\tau$ is given as $\mathbf{y} = \text{null}(\mathbf{C}(\tau))$. Finally, fundamental matrix $\mathbf{F}$ regarding to each obtained focal length can be directly estimated using Eq.~\ref{eq:null_space}. 

\begin{table}[h!]
  \centering
  \resizebox{\columnwidth}{!}{ \begin{tabular}{ | c | c c c c c c c c c c | }	
  \hline
        \multirow{2}{*}{$\mathbf{C}(\tau)$} & 1 & 2 & 3 & 4 & 5 & 6 & 7 & 8 & 9 & 10 \\ 
  		& $\alpha^3$ & $\alpha^2 \beta$ & $\alpha^2 \gamma$  & $\alpha \beta^2$ & $\alpha \beta \gamma$ & $\alpha \gamma^2$ & $\beta^3$ & $\beta^2 \gamma$ & $\beta \gamma^2$ & $\gamma^3$ \\
  \hline
        1 & $c_1$ & $c_2$ & $c_3$ & $c_4$ & $c_5$ & $c_6$ & $c_7$ & $c_8$ & $c_9$ & $c_{10}$ \\ 
        . & . & . & . & . & . & . & . & . & . & .\\ 
        10 & $c_{91}$ & $c_{92}$ & $c_{93}$ & $c_{94}$ & $c_{95}$ & $c_{96}$ & $c_{97}$ & $c_{98}$ & $c_{99}$ & $c_{100}$ \\ 
  \hline
  \end{tabular} }
  \caption{ The coefficient matrix $\mathbf{C}(\tau)$ related to the ten polynomial equations of the trace constraint.}
  \label{tab:coefficient_table}
\end{table}

\section{Elimination and Selection of Roots}

In this section, a novel technique is proposed to omit roots on the basis of the underlying geometry. Then we show a heuristics considering the properties of digital cameras to remove invalid focal lengths. In the end, we introduce a root selection algorithm. \vspace{5px}

\noindent \vspace{2px}
\textbf{4.1. Elimination of Invalid Focal Lengths}

A solution is proposed here based on the underlying geometry to eliminate invalid focal lengths. Suppose that a point pair $(\mathbf{p}_1, \mathbf{p}_2)$, the related local affinity $\mathbf{A}$, the fundamental matrix $\mathbf{F}$, and an obtained focal length $f$ are given. As the semi-calibrated case is assumed, $\mathbf{F}$ and $f$ exactly determines the projection matrices $\mathbf{P}_1$ and $\mathbf{P}_2$ of both cameras~\cite{Hartley2003}. Denote the 3D coordinates and the surface normal induced by point pair $(\mathbf{p}_1, \mathbf{p}_2)$, local affinity $\mathbf{A}$ and the projection matrices with $\mathbf{q} = [x \quad y \quad z]^T$ and $\mathbf{n} = [n_x \quad n_y \quad n_z]^T$, respectively. According to our experiences, linear triangulation~\cite{Hartley2003} is a suitable and efficient choice to estimate $\mathbf{q}$. Surface normal $\mathbf{n}$ is estimated exploiting affinity $\mathbf{A}$ by the method proposed in~\cite{Barath2015}.\footnote{\url{http://web.eee.sztaki.hu/~dbarath/}} 

Without loss of generality, we assume that a point of a 3D surface cannot be observed from behind. As a consequence, the angle between vectors $\mathbf{c}_{i} - \mathbf{q}$ and $\mathbf{n}$ must be smaller than $90^{\circ}$ for both cameras, where $\mathbf{c}_{i}$ is the position of the $i$-th camera ($i \in \{1,2\}$). This can be interpreted as follows: each camera selects a half unit-sphere around the observed point $\textbf{q}$. Surface normal $\textbf{n}$ must lie in the intersection of these half spheres. 
%
These half spheres are described by a rectangle in the spherical coordinate system as follows: 
$
	\textrm{rect}_{i} = \begin{bmatrix} \theta_{i} - \frac{\pi}{2} & \sigma_{i} - \frac{\pi}{4} & \pi & \frac{\pi}{2} \end{bmatrix}
$, 
 where $\theta_{i}$, $\sigma_{i}$ are the corresponding spherical coordinates and $\textrm{rect}_{i}$ is of format $\begin{bmatrix} \textrm{corner}_\theta & \textrm{corner}_\sigma & \textrm{width} & \textrm{height} \end{bmatrix}$. The intersection area induced by the two cameras is as 
\begin{equation*}
	\textrm{rect}_\cap = \bigcap_{i \in [1,2]} \textrm{rect}_{i}.
\end{equation*}

Point $\textbf{q}$ is observable from both cameras \textit{if and only if} surface normal $\textbf{n}$, represented by spherical coordinates $\Theta$ and $\Sigma$, lies in the intersection area: $\begin{bmatrix} \Theta & \Sigma \end{bmatrix} \in \textrm{rect}_\cap$. A setup, induced by focal length $f$, not satisfying this criteria is an invalid one and can be omitted. Note that this constraint can be straightforwardly extended to the multi-view case making the intersection area more restrictive.  \vspace{5px}

\noindent \vspace{2px}
\textbf{4.2. Physical Properties of Cameras}

We introduce restrictions on the estimated roots considering the physical limits of the cameras. The focal length within camera matrix $\mathbf{K}$ is not equivalent to the focal length of the lenses, since it is the ratio of the optical focal length and the pixel size~\cite{Hartley2003}. Particularly, the latter one is a few micrometers, while the optical focal length are within interval $[1 \dots 500]$ mm. Therefore, coarse lower and upper limits for a realistic camera are $100$ and $500.000$. Focal lengths out of this interval are automatically discarded. Note that these limits can be easily changed considering cameras with different properties. \vspace{5px}

\noindent \vspace{2px}
\textbf{4.3. Root Selection}

To resolve the ambiguity of multiple roots and to minimize the effect of the noise, the classical way is to exploit multiple measurements eliminating the inconsistent ones. Since Eq.~\ref{eq:determinant_hidden} is a high-degree polynomial it is sensitive to noise -- small changes in the coordinates and affine elements cause significantly different coefficients.

RANSAC~\cite{RANSAC} is a successful technique for that problem, e.g. in the five-point relative-orientation one~\cite{nister2004efficient}. Recent methods, i.e. Kernel Voting, exploit the property that the roots form a peak around the real solution~\cite{li2005non,li2006simple,kukelova2012algebraic}. Kernel Voting maximizes a kernel density function like a maximum-likelihood-decision-maker. To our experiences, this technique works accurately if the noise in the coordinates does not exceed $1-2$ pixels on average. Over that, the roots may form several strongly supported peaks and it is not guaranteed that the true solution is found. 

Thus we formulate the problem as a mode-seeking in a one dimensional domain:  the real focal length appears as the most supported mode. Among several mode-seeking techniques~\cite{jain1999data} the most robust one is the Median-Shift~\cite{shapira2009mode} according to extensive experimentation. Median-Shift providing Tukey-medians~\cite{tukey1975mathematics} as modes does not generate new elements in the domain it is applied to. In particular, there is no significant difference in the results of Tukey-~\cite{tukey1975mathematics} and Weiszfeld-medians~\cite{weiszfeld1937point}, however, the former one is slightly faster to compute. Finally, in order to overcome the discrete nature of Median-Shift -- since it does not add new instances, only operates with the given ones --, we apply a gradient descent from the retrieved mode $x_0$ maximizing function
\begin{equation}
	f(x) = \sum_{i = 1}^{n} \frac{\kappa(x_i - x)}{h},
\end{equation}
where $n$ is the number of focal lengths, $\kappa$ is a kernel function -- we chose Gaussian-kernel --, $x_i$ is the $i$-th focal length, and $h$ is a bandwidth same as for the Median-Shift. 

\section{Experimental Results}

For the synthesized tests, we used the MATLAB code shown in Alg.~\ref{MatlabAlgorithm}. For the real world tests, we used our C++ implementation\footnote{\url{http://web.eee.sztaki.hu/~dbarath/}} which is a modification of the solver of Hartley et al.~\cite{hartley2012efficient}.
\noindent \vspace{5px}


\noindent \vspace{2px}
\textbf{5.1. Synthesized tests}

For synthesized testing, two perspective cameras are generated by their projection matrices $\mathbf{P}_1$ and $\mathbf{P}_2$. The first camera is at position $[0 \; 0 \; 1]^T$ looking towards the origin, and the distance of the second one from the first is $0.15$ in a random direction. Five random planes passing over the origin are generated and each is sampled in fifty random locations. The obtained 3D points are projected onto the cameras. Zero-mean Gaussian-noise is added to the point coordinates.
The local affine transformations are calculated by derivating the homographies induced by the tangent planes at the noisy point correspondences similarly to~\cite{barathnovel}.



Figure~\ref{fig:root_number} reports the kernel density function with Gaussian-kernel width $10$ plotted as the function of the relative error (in percentage). Candidate focal lengths are estimated as follows:\vspace{3px}

\noindent \hspace{8px} \vspace{0.5px}
\textbf{1.} Select two affine correspondences.

\noindent \hspace{8px} \vspace{0.5px}
\textbf{2.} Apply the proposed 2-point method. 

\noindent \hspace{8px}
\textbf{3.} Repeat from Step 1.\vspace{3px}
    
The iteration limit is chosen to $100$. The blue horizontal line reports the result of Median-Shift, the green one is that of Kernel Voting. The $\sigma$ value of the zero-mean Gaussian-noise added to the point locations and affinities is (a) $0.01$ pixels, (b) $0.1$ pixels, (c) $1.0$ pixels, (d) $3.0$ pixels, (e) $3.0$ pixels and there are $10\%$ outliers, (f) $1.0$ pixels with some errors in the aspect ratio: the true one is $1.00$ but $0.95$ is used. The real focal length is $600$. 

Confirming the validity of the proposed theory, the peak is over the ground truth focal length: $0\%$ relative error. The proposed root selection is more robust than the Kernel Voting approach since the blue line is closer to the zero relative error even if the noise is high. 
 
\begin{figure*}[htbp]
    \centering
    \begin{subfigure}[b]{0.33\textwidth}
        	\includegraphics[width = 0.999\columnwidth]{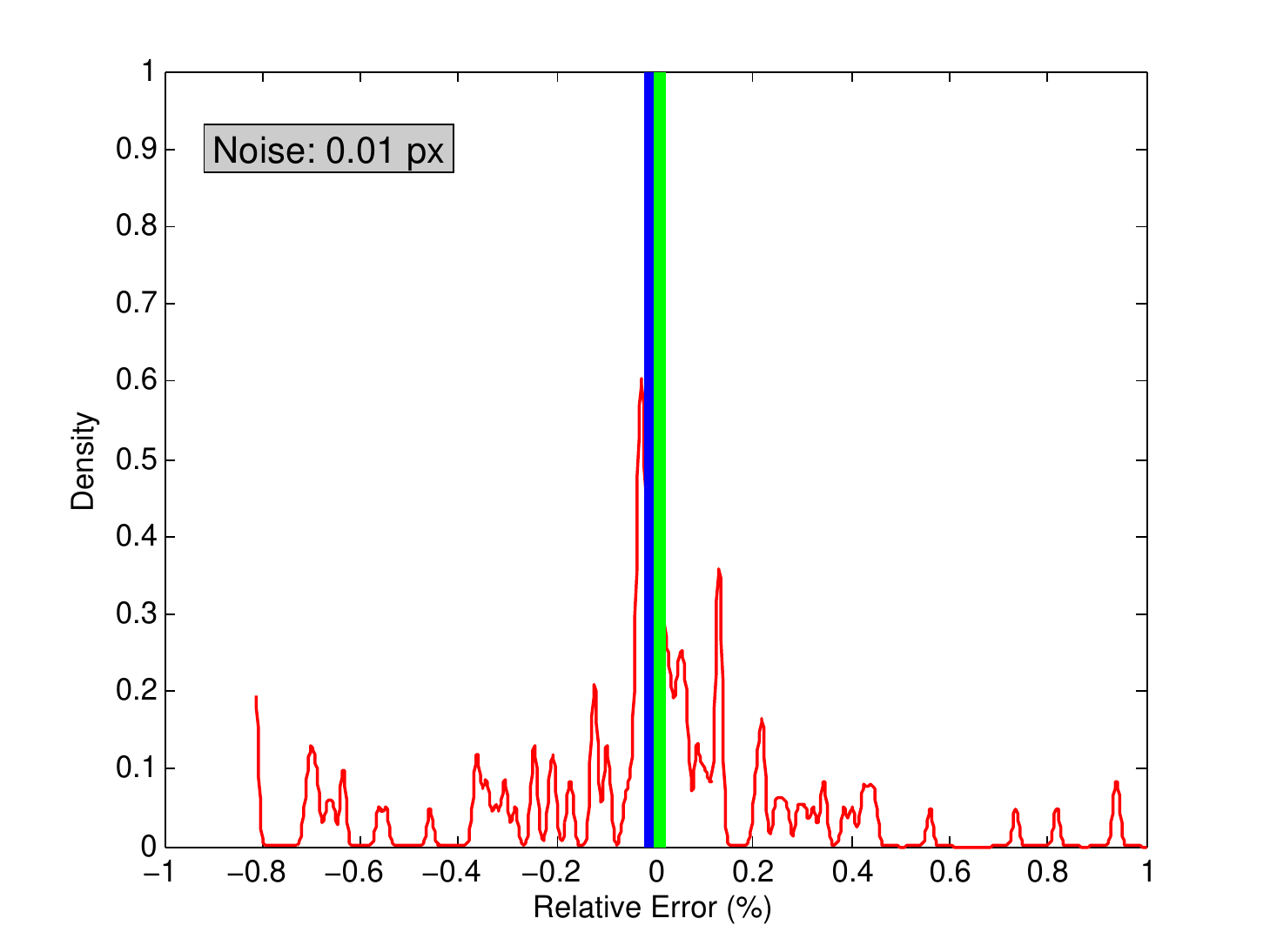}
           	\caption{}
    \end{subfigure} 
    \begin{subfigure}[b]{0.33\textwidth}
        	\includegraphics[width = 0.999\columnwidth]{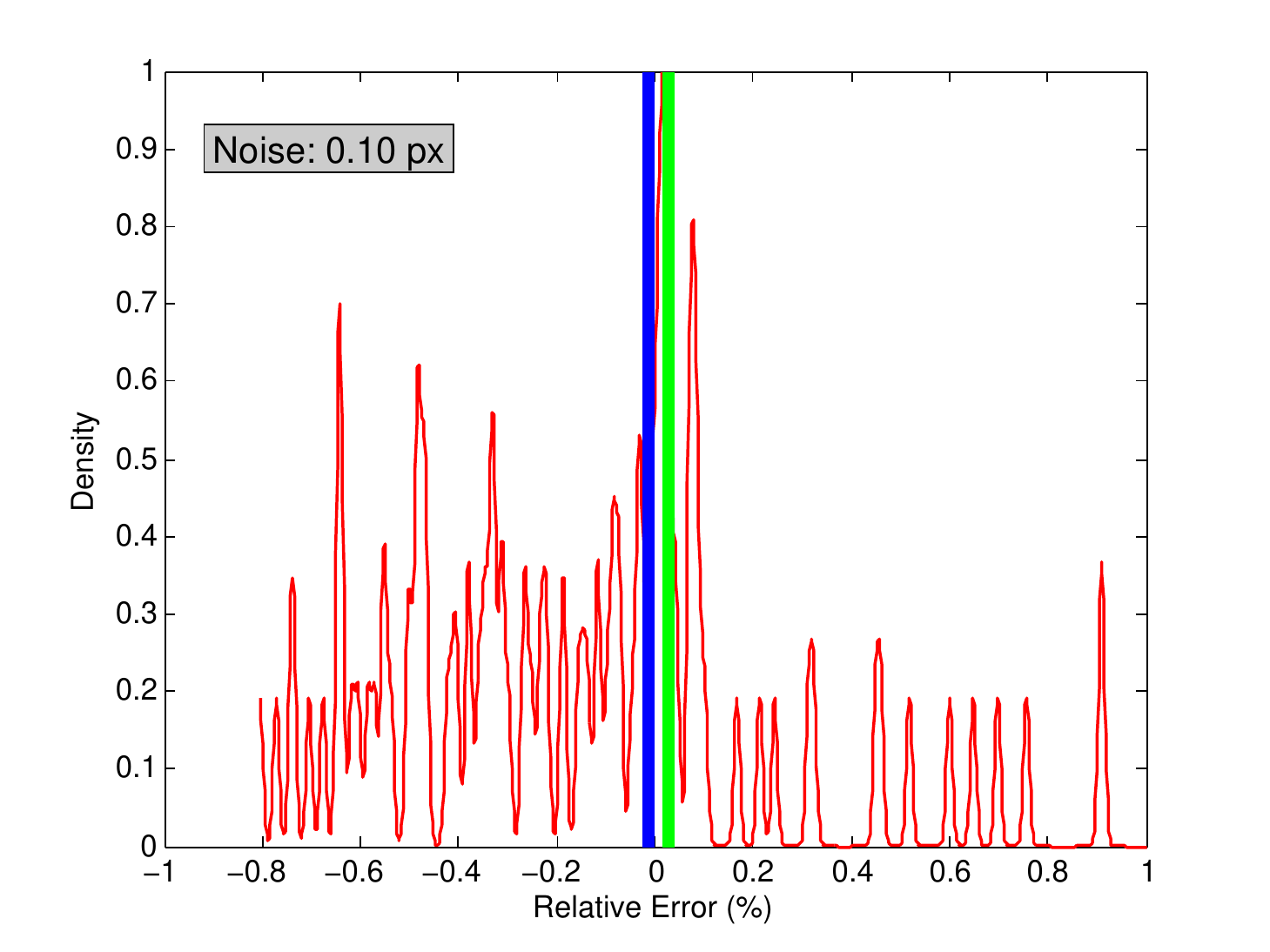}
           	\caption{}
    \end{subfigure} 
    \begin{subfigure}[b]{0.33\textwidth}
        	\includegraphics[width = 0.999\columnwidth]{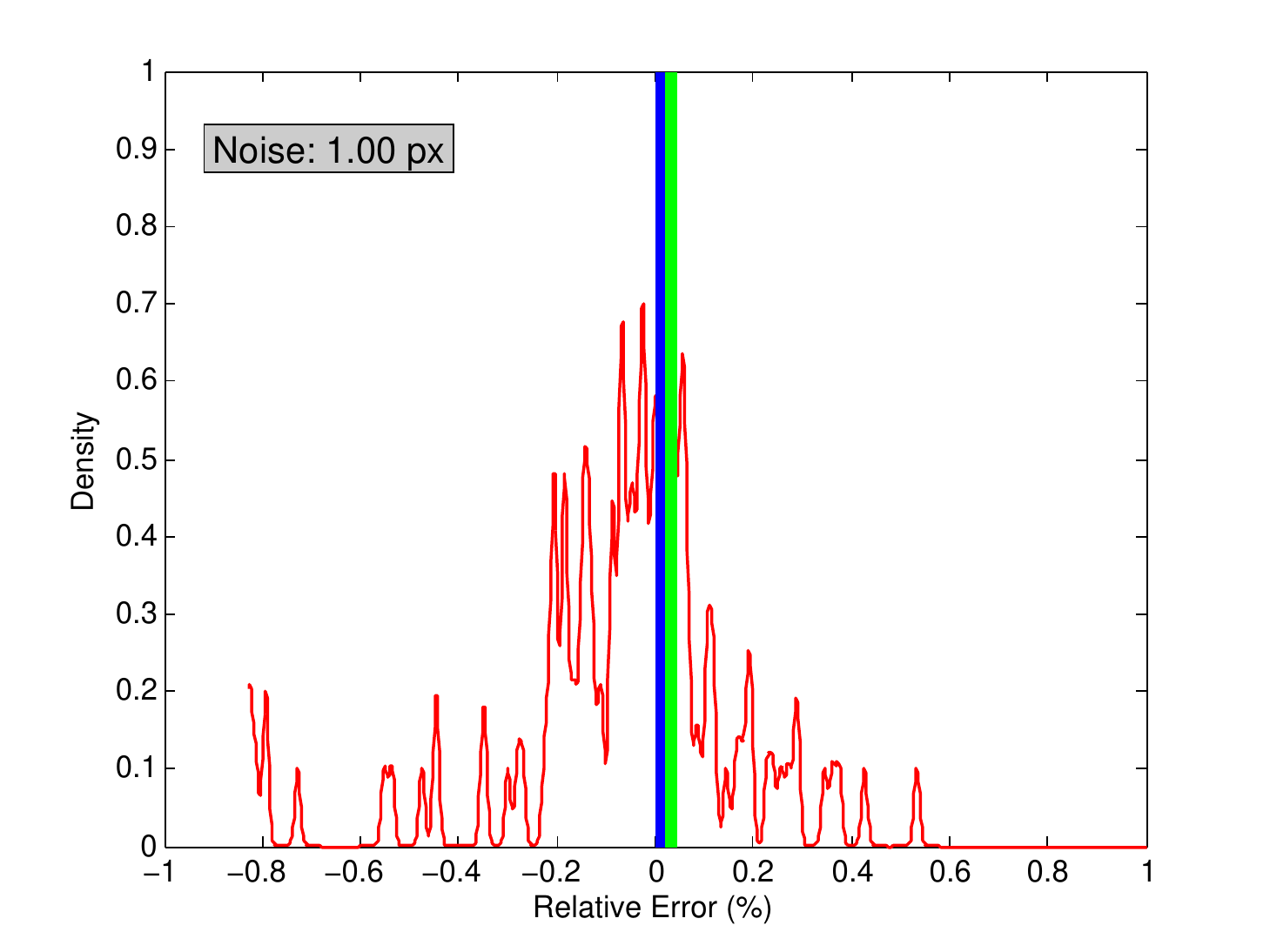}
           	\caption{}
    \end{subfigure}\\
    \begin{subfigure}[b]{0.33\textwidth}
        	\includegraphics[width = 0.999\columnwidth]{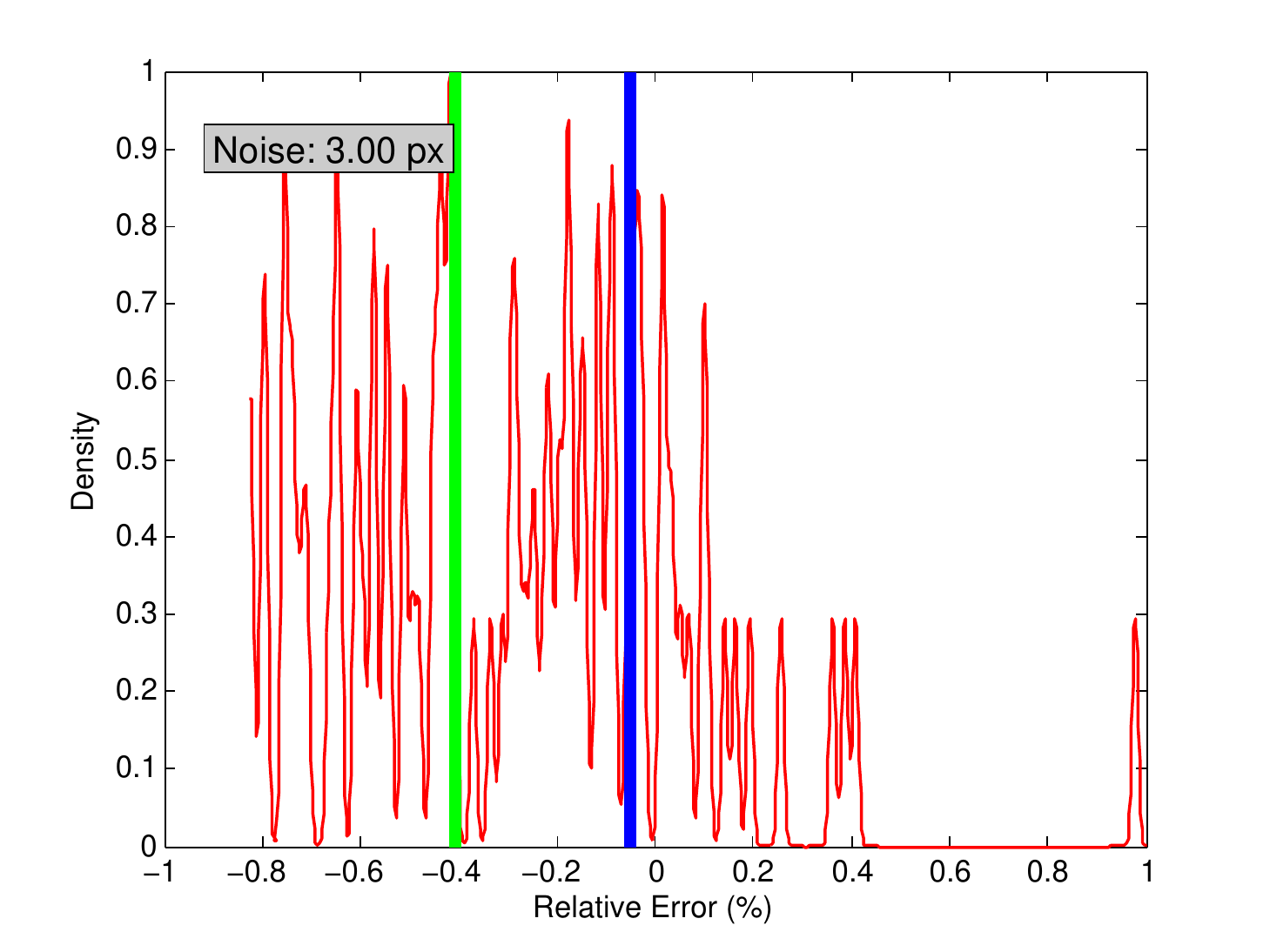}
           	\caption{}
    \end{subfigure}      
    \begin{subfigure}[b]{0.33\textwidth}
        	\includegraphics[width = 0.999\columnwidth]{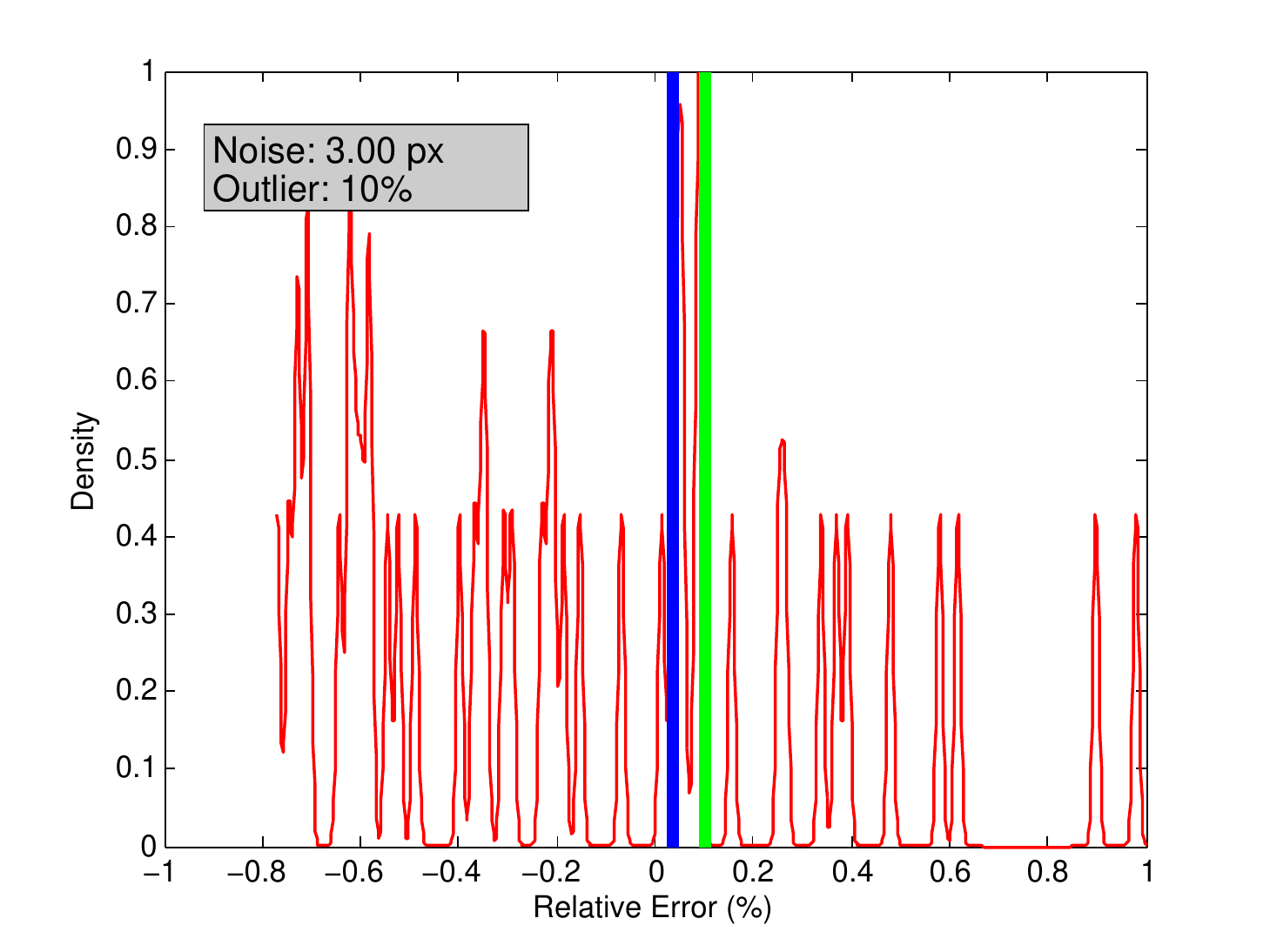}
           	\caption{}
    \end{subfigure}      
    \begin{subfigure}[b]{0.33\textwidth}
        	\includegraphics[width = 0.999\columnwidth]{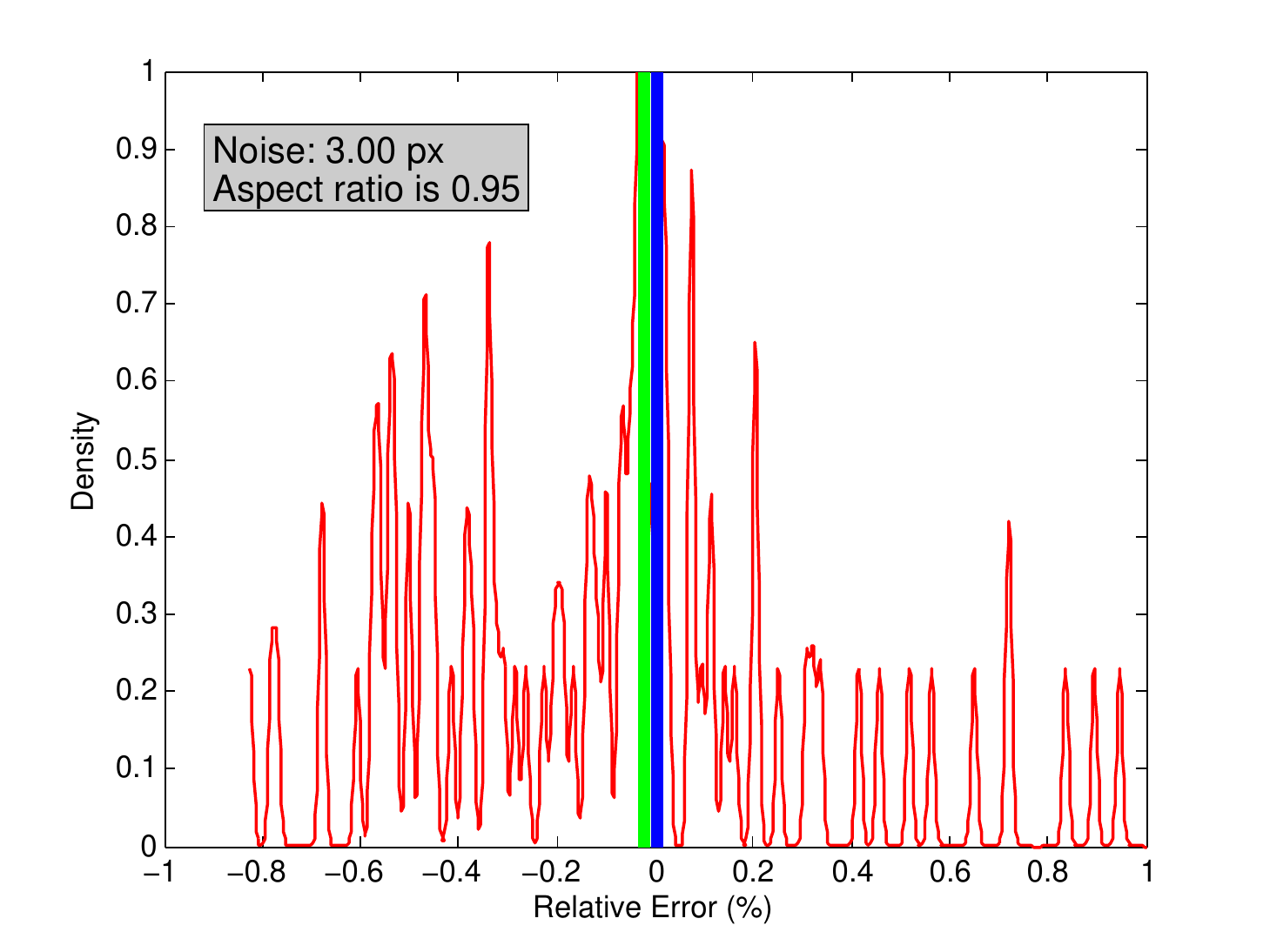}
           	\caption{}
    \end{subfigure}        
    \caption{ The kernel density function (vertical axis) with Gaussian-kernel width $10$ plotted as the function of the relative error ($\%$). Five planes are generated and each is sampled in $20$ locations -- points are projected onto the cameras and local affinities are calculated. The blue horizontal line is the result of Median-Shift, the green one is that of the Kernel Voting. The $\sigma$ value of the zero-mean Gaussian-noise added to the point locations and affinities is (a) $0.01$ pixels, (b) $0.1$ pixels, (c) $1.0$ pixels, (d) $3.0$ pixels, (e) $3.0$ pixels and there are $10\%$ outliers, (f) $1.0$ pixels with some errors in the aspect ratio: the true one is $1.00$ but $0.95$ is used. Ground truth focal length is $600$. Best viewed in color.}
    \label{fig:root_number}
\end{figure*} 
\vspace{5px}

Fig.~\ref{fig:fundamental_error} reports the mean (top) and median (bottom) errors of the estimated fundamental matrices plotted as the function of the noise $\sigma$ and compared with the results of Hartley et al.\cite{hartley2012efficient} and Perdoch et al.\cite{PerdochMC06}. The error is the Frobenious norm of the estimated and ground truth fundamental matrices. $100$ runs were performed on each noise level. It can be seen that the accuracy of the estimated fundamental matrices is similar to that of Hartley et al.~\cite{hartley2012efficient}. 
\noindent \vspace{5px}

\begin{figure}[htbp]
    \centering
     \includegraphics[width = 0.70\columnwidth]{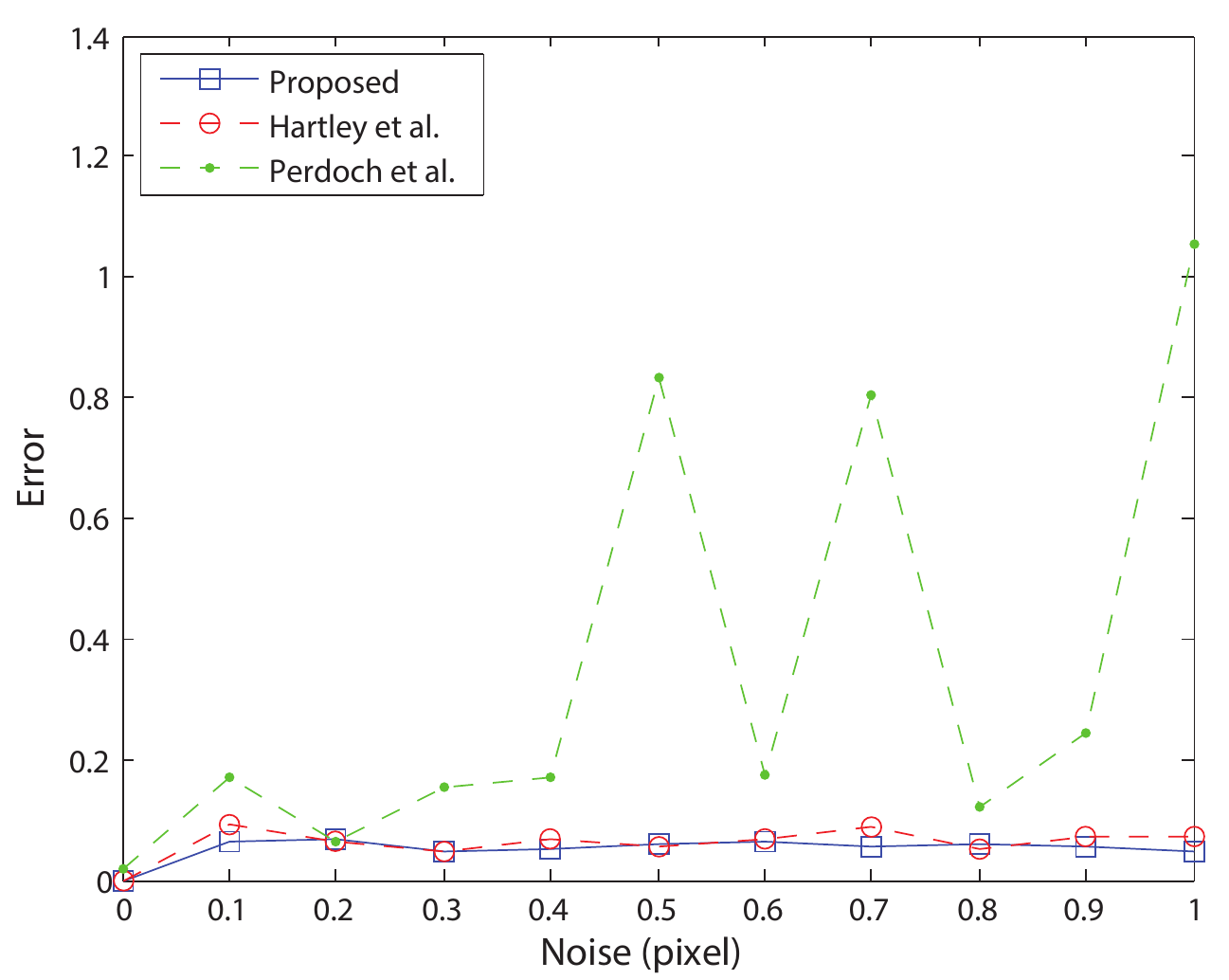}
     \includegraphics[width = 0.70\columnwidth]{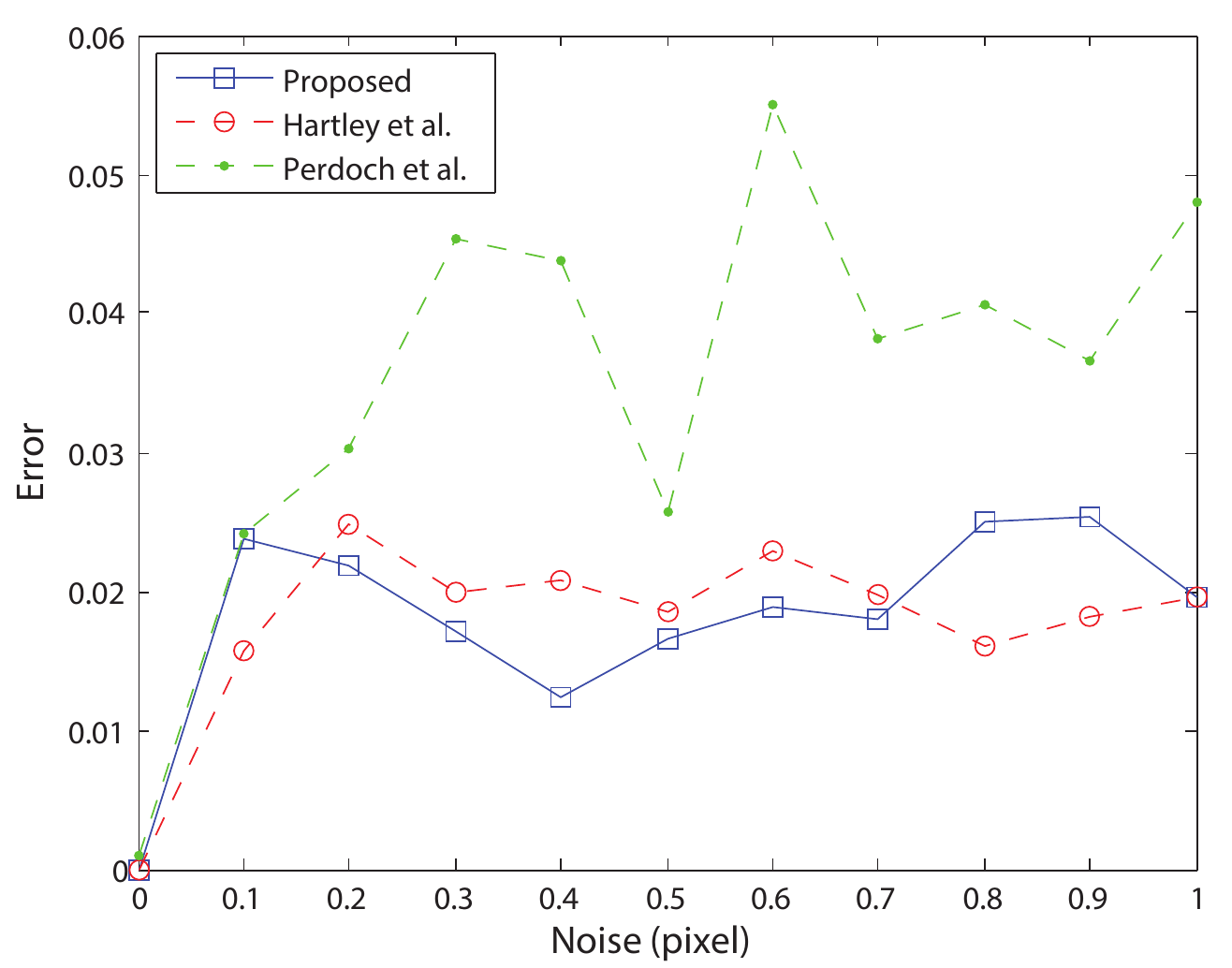}
    \caption{The mean (top) and median (bottom) Frobenious norms of the estimated and the ground truth fundamental matrices plotted as the function of the noise $\sigma$. 100 runs on each noise level were performed.}
    \label{fig:fundamental_error}
\end{figure} 
\vspace{5px}

\noindent \vspace{2px}
\textbf{5.2. Tests on Real Images\footnote{Test data are provided as supplemental material.}}

To test the proposed method on real world photos, $104$ image pairs were downloaded\footnote{\url{http://www2c.airnet.ne.jp/kawa/photo/ste-idxe.htm}} each containing the ground truth focal length in the EXIF data (see Fig.~\ref{fig:example_pairs} for examples). Affine correspondences are detected by ASIFT~\cite{Morel2009} and the same procedure is applied as for the synthesized tests. Fig.~\ref{fig:real_focals} reports the histogram of the relative errors (in percentage) in the focal length estimates on all the $104$ pairs. It can be seen that in most of the cases the obtained results are accurate, the relative error is close to zero. Fig.~\ref{fig:real_image} shows the first image of an example pair and the point correspondences. 

In Table~\ref{tab:comparison1}, the proposed method is compared with the 6-point algorithm~\cite{hartley2012efficient} and the one creating point correspondences from two local affinities~\cite{PerdochMC06}. The reported relative errors are computed as the ratio of the estimation error and the ground truth focal length as $|f_{est} - f_{gt}|/f_{gt}$. It can be seen that the 2-point technique outperforms the other ones in terms of both mean and median accuracy and spread.


\begin{table}
\caption{ Mean (Avg) and median (Med) relative error (in percentage) and the spread ($\sigma$) of the relative errors in the estimated focal lengths on the $104$ real image pairs. Corr \# denotes the required correspondence number.}
\center
  	\begin{tabular}{| l || c | c c c |}
    \hline
 	 	Method & Corr \# & Avg & Med & $\sigma$ \\ 
    \hline
 	 	\textbf{Proposed} & 2 & \phantom{x}\textbf{9.62} & \phantom{x}\textbf{3.88} & \textbf{14.08} \\
 	 	Perdoch et al.~\cite{PerdochMC06} & 2 & 44.66 & 45.89 & 26.43 \\
 	 	Hartley et al.~\cite{hartley2012efficient} & 6 & 21.79 & \phantom{x}8.61 & 27.48 \\
	\hline
\end{tabular}
\label{tab:comparison1}
\end{table}

\begin{figure}[htbp]
    \centering
    \begin{subfigure}[b]{0.50\columnwidth}
        	\includegraphics[width = 1.0\columnwidth]{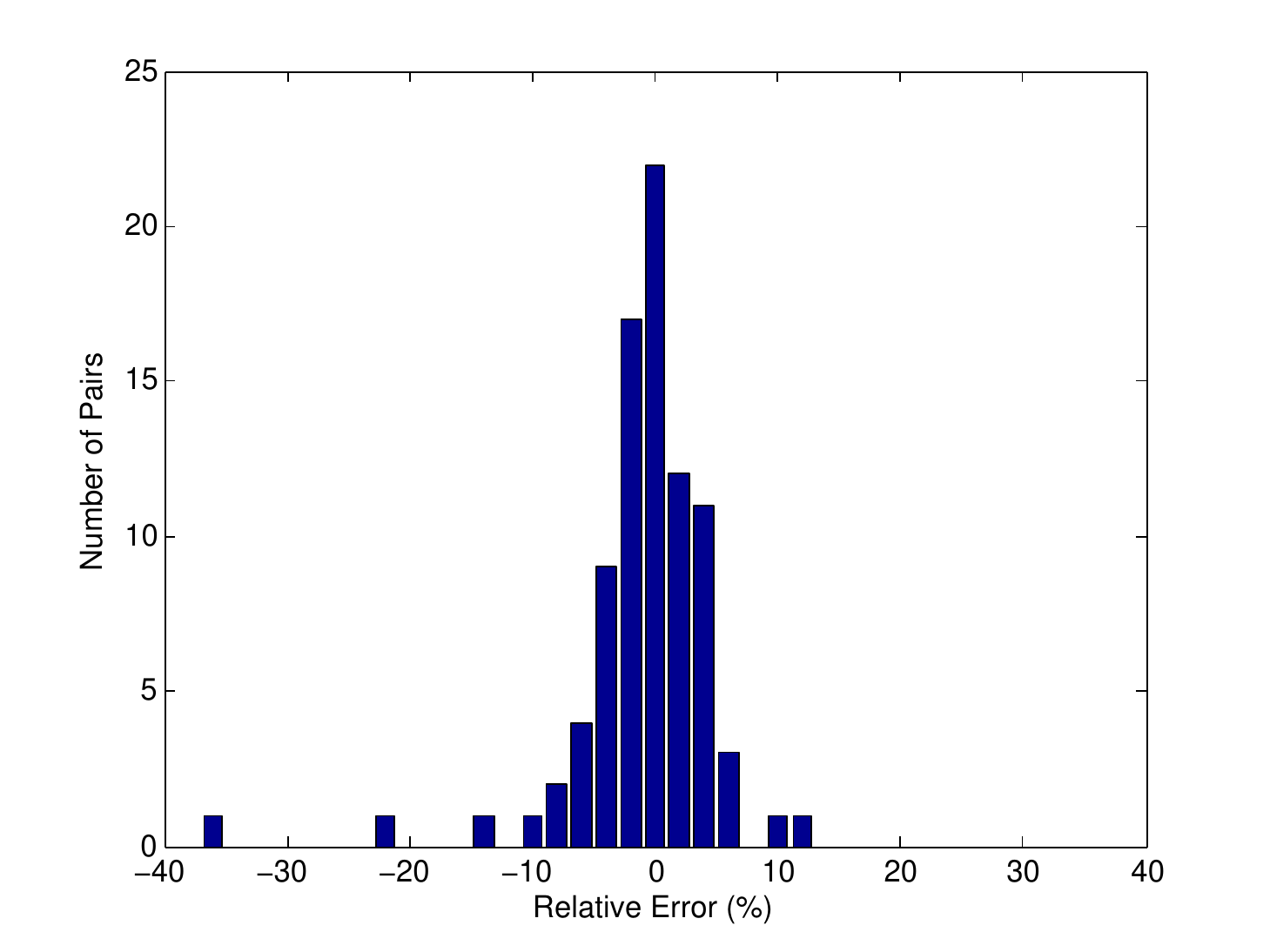}
    		\caption{}
      		\label{fig:real_focals}
    \end{subfigure}    
    \begin{subfigure}[b]{0.46\columnwidth}
        	\includegraphics[width = 1.0\columnwidth]{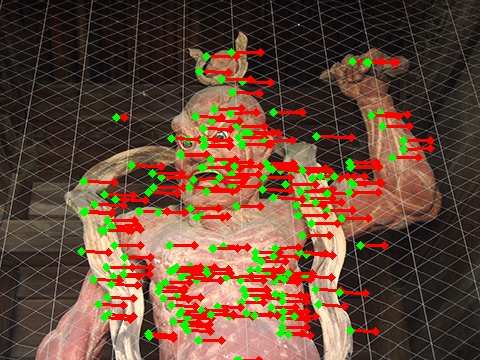}
    		\caption{}
      		\label{fig:real_image}
    \end{subfigure}        
    \caption{ (a) Histogram of focal length estimation on 104 image pairs. The horizontal axis is the number of the pairs plotted as the function of the relative error ($\%$, vertical axis) in the focal length. (b) The first image of an example pair. Point coordinates on the first image (green dots), on the second one (red dots) and the point movements (red lines). }
\end{figure} \vspace{5px}

\noindent \vspace{2px}
\textbf{5.3. Time Demand}

Augmenting RANSAC or other robust statistics with the proposed method significantly reduces the processing time. Table~\ref{tab:ransac_iterations} reports the required iteration number~\cite{Hartley2003} of RANSAC to converge using different minimal methods (columns) as engine. Rows show the ratio of the outliers. 

\begin{table}[h!]
  \centering
  \caption{Required iteration number of RANSAC augmented with minimal methods (columns) with $95\%$ probability on different outlier levels (rows).}
  \begin{tabular}{ | c | r r r r r r | }
  \hline
      & \multicolumn{5}{c}{\# of required points} & \\ 
    Outl. & \textbf{2} & 5\phantom{x} & 6\phantom{x} & 7\phantom{x} & 8\phantom{x} & \\ 
  \hline
    50\% & \textbf{11} & $95${\footnotesize\phantom{x}} & $191${\footnotesize\phantom{x}} & $383${\footnotesize\phantom{x}} & $766${\footnotesize\phantom{x}} & \\ 
    80\% & \textbf{74} & $\sim10^3$ & $\sim10^4$ & $\sim10^5$ & $\sim10^6$
 & \\
  \hline
  \end{tabular}
  \label{tab:ransac_iterations}
\end{table}

\begin{figure}[htbp]
    \centering
     \includegraphics[width = 0.49\columnwidth]{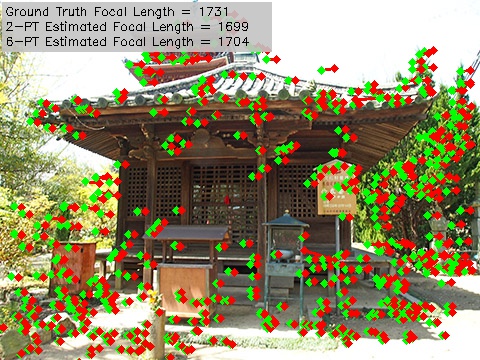}
     \includegraphics[width = 0.49\columnwidth]{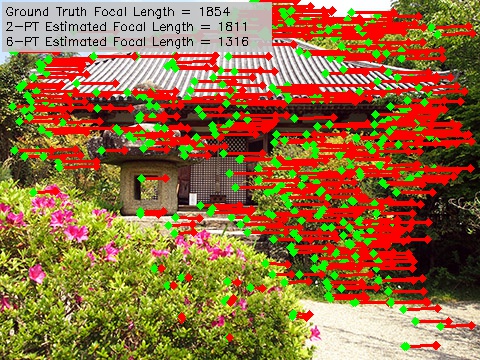}
     \includegraphics[width = 0.49\columnwidth]{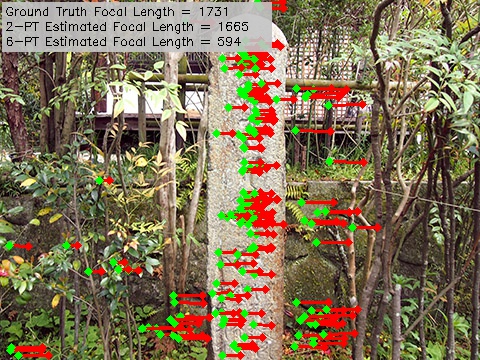}
     \includegraphics[width = 0.49\columnwidth]{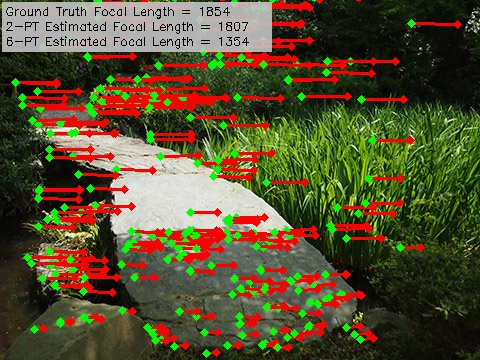}
     \includegraphics[width = 0.49\columnwidth]{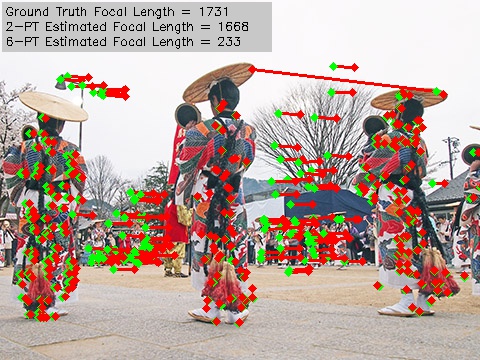}
     \includegraphics[width = 0.49\columnwidth]{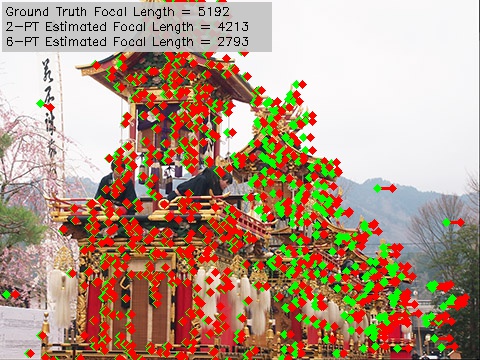}
    \caption{The first images of example pairs. Point coordinates on the first image (green dots), on the second one (red dots) and the point movements (red lines). The ground truth focal lengths, the results of the 6-point~\cite{hartley2012efficient} and the proposed methods are written in gray rectangle.}
    \label{fig:example_pairs}
\end{figure} 
\vspace{5px}

\section{Conclusion}

A theory and an efficient method is proposed to estimate the unknown focal-length and the fundamental matrix using only two affine correspondences. The 2-point method is validated on both synthesized and real world data. Compared with the state-of-the-art methods, it obtained the most accurate focal lengths with fundamental matrices having similar quality as the recent algorithms. Combining the minimal solver with a robust statistics, e.g.\ RANSAC, allows significant reduction in computation. Particularly, its time demand is around a few milliseconds, thus it is much faster than affine-covariant detectors providing the input.

The proposed algorithm can also be applied in reconstruction or multi-view pipelines, e.g.\ that of Bujnak et al.~\cite{bujnak2010robust}, if at least two images of the same camera with fixed focal length are available.


\appendix

\begin{figure}
	\centering
    
    \begin{subfigure}{.9\columnwidth}
  		\centering
		\includegraphics[width=1.0\columnwidth]{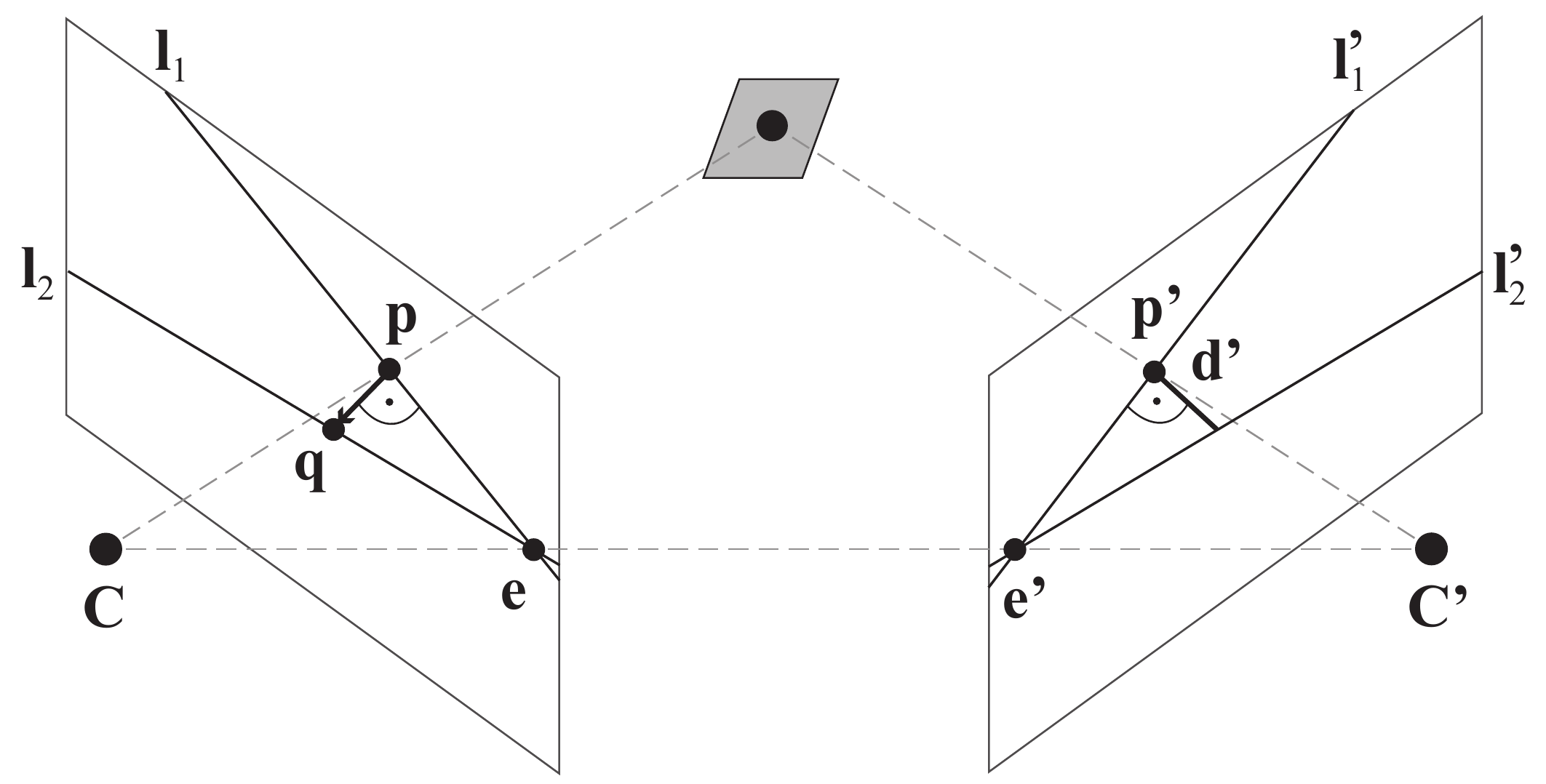}
  		\caption{The scale between neighboring epipolar lines.}
  		\label{fig:stereo_lines}
	\end{subfigure}
\caption{ Two projections of a patch. The constraint for scale states that the ratio of $|p - q|$ and $d'$ determines the scale between vectors $\mathbf{A}^{-T} \mathbf{n}$ and $\mathbf{n}'$. }
\label{fig:proof}
\end{figure}

\section{Proof of the Linear Affine Constraints}
\label{appendix:proof_affine_constraints}

\begin{mylem}[Constraints on the Normals of Epipolar Lines]
Given a local affine transformation $\mathbf{A}$ transforming the infinitely close vicinities of the related point pair. The normals of the corresponding epipolar lines are $\mathbf{n}_1$ and $ \mathbf{n}_2$. Matrix $\mathbf{A}$ is a valid local affinity if and only if $\mathbf{A}^{-T} \mathbf{n}_1 = -\mathbf{n}_2$. 
\end{mylem}

\begin{proof}
It is trivial that affinity $\mathbf{A}$ transforms the direction of the corresponding epipolar lines to each other as $\mathbf{A} \mathbf{v} \parallel \mathbf{v}'$, where $\mathbf{v}$ and $\mathbf{v}'$ are the directions of the lines on the two images. It is well-known from Computer Graphics~\cite{turkowski1990transformations} that this is equivalent to $\mathbf{A}^{-T} \mathbf{n} = \beta \mathbf{n}'$, where $\mathbf{n} = (\mathbf{F}^T \mathbf{p}')_{1:2}$ and $\mathbf{n}' = (\mathbf{F} \mathbf{p})_{1:2}$ are the normals of the epipolar lines ($\beta \not= 0$). Note that lower index $(1:2)$ denotes the first two elements of a vector. We prove here that 
\begin{equation}
	\mathbf{A}^{-T} \mathbf{n} = -\mathbf{n}'.
	\label{eq:initial_norm}
\end{equation}

\noindent
\textbf{(Proof)} Given a corresponding point pair $\mathbf{p}=[x, y, 1]^T$ and $\mathbf{p}'=[x', y', 1]^T$. Let $\mathbf{n}_1 = [\mathbf{n}_{1,x} \quad \mathbf{n}_{1,y}]^T$ and $\mathbf{n}_1' = [\mathbf{n}_{1,x}' \quad \mathbf{n}_{1,y}']^T$ be the normal directions of epipolar lines $\mathbf{l}_1 = \mathbf{F}^T \mathbf{p}' = [\mathbf{l}_{1,a} \quad \mathbf{l}_{1,b} \quad \mathbf{l}_{1,c}]^T$ and $\mathbf{l}_1' = \mathbf{F} \mathbf{p} = [\mathbf{l}_{1,a}' \quad \mathbf{l}_{1,b}' \quad \mathbf{l}_{1,c}']^T$. Then it is trivial that $\mathbf{A}^{-T} \mathbf{n}_1 = \beta \mathbf{n}_1'$ due to $\mathbf{A} \mathbf{v} \parallel \mathbf{v}'$, where $\beta$ is a scale factor.

First, the task is to determine how affinity $\mathbf{A}$ transforms the length of $\mathbf{n}_1$ if $| \mathbf{n}_1 | = | \mathbf{n}_1' | = 1$. 
Introduce point $\mathbf{q} = \mathbf{p} + \delta \mathbf{n}_1$, where $\delta$ is an arbitrary scalar value. This new point determines an epipolar line on the second image as $\mathbf{l}_2'  = \mathbf{F} \mathbf{q} = \mathbf{F} (\mathbf{p} + \delta \mathbf{n}_1) = [\mathbf{l}_{2,a}' \quad \mathbf{l}_{2,b}' \quad \mathbf{l}_{2,c}']^T$.
Scale $\beta$ is given by distance $d'$ between line $\mathbf{l}_2'$ and point $\mathbf{p}'$ (see Fig.~\ref{fig:stereo_lines}). The calculation of distance $d'$ is written as follows:
\begin{eqnarray}
		\label{eq:dist_2}
		d' &= \frac{|s_{1,a} x' + s_{2,b} y' + s_{3,c} |}{ \sqrt{s_{1,a}^2 + s_{2,b}^2} }, \\
		s_{i,k} &= \mathbf{l}_{1,k}' + \delta f_{i1} \mathbf{n}_{1,x} + \delta f_{i2} \mathbf{n}_{1,y}, \nonumber \\
        i &\in \{1,2,3\}, k \in \{a,b,c\}\nonumber 
\end{eqnarray}
Point $\mathbf{p}'$ lies on $\mathbf{l}_1'$, which can
be written as $\mathbf{l}_{1,a}' x' + \mathbf{l}_{1,b}' y' + \mathbf{l}_{1,c}' = 0$. This fact reduces Eq.~\ref{eq:dist_2} to 
\begin{eqnarray}
		\label{eq:dist_3}
		d' = \frac{| \hat{s}_1 u' + \hat{s}_2 v^2 + \hat{s}_3|}{ \sqrt{s_1^2 + s_2^2} },
\end{eqnarray}
where $\hat{s}_{i} = \delta f_{i1} \mathbf{n}_{1,x} + \delta f_{i2} \mathbf{n}_{1,y}, \quad i \in \{1,2,3\}$.
%
To determine $\beta$, the introduced point $\mathbf{q}$ has to be moved infinitely close to $\mathbf{p}$ ($\delta \to 0$). The square of $\beta$ is then written as $
%
		\label{eq:dist_4}
		\beta^2 = \lim_{\delta \to 0} \frac{\delta^2}{d'^2} = \lim_{\delta \to 0} \frac{s_1^2 + s_2^2}{ | \hat{s}_1 u' + \hat{s}_2 v' + \hat{s_3}|^2 }.
$
%
After elementary modifications, the formula for scale $\beta$ is 
$
	\beta = \sqrt{\mathbf{l}_{1,a}' \mathbf{l}_{1,a}' + \mathbf{l}_{1,b}' \mathbf{l}_{1,b}'} / \left( \left| \widetilde{s_1} x' + \widetilde{s_2} y' + \widetilde{s_3} \right| \right),
$
where 
$
	\widetilde{s}_i = f_{i1} \mathbf{n}_{1,x} + f_{i2} \mathbf{n}_{1,y}, \; i \in \{1,2,3\}
$. 
Therefore, \textit{we can calculate $\beta$ for unit length normals}. 

Consider the case when normals are kept in their original form and not normalized ($|\mathbf{n}_1| \neq |\mathbf{n}_1'| \neq 1$). The normalization indicates the following formula
\begin{equation}
	\label{eq:normalization_proof_1}
	\mathbf{A}^{-T} \frac{\mathbf{n}}{| \mathbf{n} |} = \beta \mathbf{n}'.
\end{equation}
The epipolar line corresponding to point $\mathbf{p}$ is parameterized as $[\mathbf{l}_{1,a}', \mathbf{l}_{1,b}', \mathbf{l}_{1,c}'] = \mathbf{F} [x, y, 1]^T$. Therefore, its normal is as follows: 
$
	\mathbf{n}' = \begin{bmatrix}
		\mathbf{l}_{1,a}' &
		\mathbf{l}_{1,b}'
	\end{bmatrix}^T = (\mathbf{F} \begin{bmatrix}
		x' &
		y' &
		1
	\end{bmatrix}^T)_{(1:2)}.
$
Similarly,
$
	\mathbf{n} = (\mathbf{F}^T \begin{bmatrix}
		x' &
		y' &
		1
	\end{bmatrix}^T)_{(1:2)}.
$
The denominator in Eq.~\ref{eq:normalization_proof_1} for computing $\beta$ is rewritten as
$
	| \mathbf{n} | = \sqrt{\mathbf{l}_{1,a}^2 + \mathbf{l}_{1,b}^2}
$.
The numerator is as follows:
\begin{eqnarray*}
	\widetilde{s_1} u' + \widetilde{s_2} v' + \widetilde{s_3} = \\
	\mathbf{n}_{1,u} (f_{11} u' + f_{21} v' + f_{31}) + \mathbf{n}_{1,v} (f_{12} u' + f_{22} v' + f_{32}) = \\
	 \mathbf{n}_{1,u}^2 + \mathbf{n}_{1,v}^2 = |\mathbf{n}_1|^2.
\end{eqnarray*}
Thus
$\beta = \pm |\mathbf{n}_{1}| /  |\mathbf{n}_{1}|^2 = \pm 1 / |\mathbf{n}_{1}|$.
Therefore, Eq.~\ref{eq:normalization_proof_1} is modified to 
$
	\mathbf{A}^{-T} \mathbf{n} = \pm \mathbf{n}'
$.

Since the direction of the epipolar lines on the two images must be the opposite of each other, the positive solution is omitted. The final formula is: $\mathbf{A}^{-T} \mathbf{n} = -\mathbf{n}'$.
\end{proof}

\begin{figure*}
\begin{minipage}{\linewidth}
\lstset{language=Matlab, basicstyle=\footnotesize, frame=single }
\begin{lstlisting}[caption=Two-point Algorithm, title=Program 1: The Two-point Algorithm, label=MatlabAlgorithm]
%%%%%%%%%%%%%%%%%%%%%%%%%%%%%%%%%%%%%%%%%%%%%%%%%%%%%%%%%%%%%%%%%%%%%%%%%%%%%%%%%%%%%%%%%%%%%%%%%%%%
%% 2-pt focal length algorithm. Use Matlab-7.0(6.5) with SymbolicMath Toolbox. 	
%% Input: The "Matches" is a 2x8 matrix containing two affine correspondences.
%% 	Each row of "Matches": (u1, v1, u2, v2, a1, a2, a3, a4).
%% 	Example (the ground truth focal length is 600):
%% Matches = [12.0527 134.0870 -263.1743 679.7212 1.6376 -0.3952 -0.1925 2.2532;
%%	-67.9281 -42.4639 -313.5657 362.3455 1.3758 -0.3845 0.0150 1.4806]
%% Output: focal lengths.
%%%%%%%%%%%%%%%%%%%%%%%%%%%%%%%%%%%%%%%%%%%%%%%%%%%%%%%%%%%%%%%%%%%%%%%%%%%%%%%%%%%%%%%%%%%%%%%%%%%%
function F = TwoPointFocalLength(Matches)
    syms F f x y z w equ Res Q C
    equ = sym('equ', [1 10]); 
    C = sym('C', [10 10]);
    Q = w^(-1) * [1, 0, 0; 0, 1, 0; 0, 0, w]; 
    M = zeros(size(pts1,1), 9); 
    
    for i = 1 : size(pts1,1)
        u1 = Matches(i,1); v1 = Matches(i,2); u2 = Matches(i,3); v2 = Matches(i,4);
        a1 = Matches(i,5); a2 = Matches(i,6); a3 = Matches(i,7); a4 = Matches(i,8);
        
        M(3*i + 0,:) = [u1 * u2, v1 * u2, u2, u1 * v2, v1 * v2, v2, u1, v1, 1];
        M(3*i + 1,:) = [u2 + a1 * u1, a1 * v1, a1, v2 + a3 * u1, a3 * v1, a3, 1, 0, 0];
        M(3*i + 2,:) = [a2 * u1, u2 + a2 * v1, a2, a4 * u1, v2 + a4 * v1, a4, 0, 1, 0];
    end;
    [~,~,vm] = svd(M,0); 
    N = [vm(:,7), vm(:,8), vm(:,9)];
    
    f = x*N(:,1) + y*N(:,2) + z*N(:,3);  
    F = transpose(reshape(f,3,3)); FT = transpose(F);  
    tr = sum(diag(F*Q*FT*Q)); 
    
    equ(1) = det(F); 
    equ(2:10) = expand(2*F*Q*FT*Q*F-tr*F);
    for i = 1:10
        equ(i) = maple('collect', equ(i), '[x,y,z]', 'distributed'); 
        for j = 1 : 10
            oper = maple('op', j, equ(i)); C(i,j) = maple('op', 1, oper);
        end
    end
    
    Res = maple('evalf', det(C)); %%Hidden-variable resultant
    foc = 1.0 ./ sqrt(double([solve(Res)])); 
    foc = foc(imag(foc) == 0);
end

\end{lstlisting}
\end{minipage}
\end{figure*}


{\small
\bibliographystyle{ieee}
\bibliography{egbib}
}

\end{document}